\documentclass[wcp]{jmlr}
\usepackage{amsmath,amssymb,graphicx,url}
\jmlrvolume{266}
\jmlryear{2025}
\jmlrworkshop{Conformal and Probabilistic Prediction with Applications}
\jmlrproceedings{PMLR}{Proceedings of Machine Learning Research}

\usepackage{hyperref}
\usepackage{url}

\usepackage{amsmath}
\usepackage{comment}
\usepackage{amscd}
\usepackage{wrapfig}
\usepackage{color, colortbl}
\usepackage{arydshln}
\usepackage{nicematrix,tikz}
\usepackage{algorithm}
\usepackage{mathtools}

\usepackage{algpseudocode}
\usepackage[capitalize,noabbrev]{cleveref}

\usepackage{amssymb}
\usepackage{dsfont}

\usepackage{amsmath,amsfonts,bm}
\usepackage{subcaption}

\def\eqref#1{equation~\ref{#1}}
\def\1{\bm{1}}

\DeclareMathAlphabet{\mathsfit}{\encodingdefault}{\sfdefault}{m}{sl}
\SetMathAlphabet{\mathsfit}{bold}{\encodingdefault}{\sfdefault}{bx}{n}

\newcommand{\T}{{\scriptscriptstyle \top}}
\title{Conformal Prediction of Classifiers with Many Classes based on Noisy Labels}
\author{
\Name{Coby Penso}\Email{coby.penso24@gmail.com}\\
\addr{Bar-Ilan University, Israel}\\
\Name{Jacob Goldberger}\Email{jacob.goldberger@biu.ac.il}\\
\addr{Bar-Ilan University, Israel}\\
\Name{Ethan Fetaya}\Email{ethan.fetaya@biu.ac.il}\\
\addr{Bar-Ilan University, Israel}}

\editor{Khuong An Nguyen, Zhiyuan Luo, Tuwe Löfström, Lars Carlsson and Henrik Boström}

\begin{document}
\maketitle
\begin{abstract}
Conformal Prediction (CP) controls the prediction uncertainty of classification systems by producing a small prediction set,   ensuring a predetermined probability that the true class lies within this set.   This is commonly done by defining a score, based on the model predictions, and setting a threshold on this score using a validation set. In this study, we address the problem of CP calibration when we only have access to a calibration set with noisy labels. We show how we can estimate the noise-free conformal threshold based on the noisy labeled data. 
We derive a finite sample coverage guarantee for uniform noise that remains effective even in tasks with a large number of classes. We dub our approach Noise-Aware Conformal Prediction (NACP). We illustrate the performance of the proposed results on several standard image classification datasets with a large number of classes.
\end{abstract}

%\begin{keywords}
%Conformal prediction, noisy labels, uniform noise 
%\end{keywords}

\section{Introduction}
In machine learning for safety-critical applications, the model must only make predictions it is confident about. One way to achieve this is by returning a (hopefully small) set of possible class candidates that contain the true class with a predefined level of certainty. This is a natural approach for medical imaging, where safety is of the utmost importance and a human makes the final decision. This allows us to aid the practitioner, by reducing the number of possible diagnoses he needs to consider, with a controlled risk of mistake.
% The practitioner now only needs to consider a small set of possibilities and can benefit greatly from the elimination of most other possibilities.
The general approach to return a prediction set without any assumptions on the data distribution (besides i.i.d. samples) is called Conformal Prediction (CP) \citep{angelopoulos2023conformal,vovk2005conformal}. It creates a prediction set with the guarantee that the probability of the correct class being within this set meets or exceeds a specified confidence threshold. The goal is to return the smallest set possible while maintaining the confidence level guarantees. Recently, with the growing use of neural network systems in safety-critical applications such as medical imaging, CP has become an important calibration tool \citep{lu2022improving,lu2022fair,olsson2022estimating}. We note that CP is a general framework rather than a specific algorithm. The most common approach builds the prediction set using a conformity score, and different algorithms mostly vary in terms of how the conformity score is defined.

When dealing with conformal predictions, a critical challenge arises in applications such as medical imaging due to label noise. In these domains, datasets frequently contain noisy labels stemming from ambiguous data that can confuse even clinical experts. Furthermore, physicians may disagree on the diagnosis for the same medical image, leading to inconsistencies in the ground truth labeling. Noisy labels also occur when applying differential privacy techniques to overcome privacy issues \citep{ghazi2021,Penso2025}. While significant efforts have been devoted to the problem of noise-robust network training \citep{9729424,Xue2022}, the challenge of calibrating the models has only recently begun to receive attention
\citep{Penso2024TMI}.

In this study, we tackle the challenge of applying CP to classification networks using a calibration set with noisy labels. \citet{einbinder2022conformal} suggested ignoring label noise and simply applying the standard CP algorithm on the noisy labeled calibration set. This strategy results in large prediction sets especially when there are many classes. A recent study suggests estimating the noise-free conformal score given its noisy version and then applying the standard CP algorithm \citep{Penso_2024}. The most related studies to ours are \citep{sesia2023adaptive} and \citep{Clarkson2024} which present a noisy CP algorithm with coverage guarantee bounds that can be too conservative in tasks with many classes (a detailed discussion on related works appears in Section 4).  Here, we present a novel algorithm for CP on noisy data that yields an effective coverage guarantee even in tasks with a large number of classes. We applied the algorithm to several standard medical and scenery imaging classification datasets and show that our method outperformed previous methods.

\section{Background}
\subsection{Conformal Prediction}

Consider a setup involving a classification network that categorizes an input $x$ into $k$ predetermined classes.  Given a coverage level of $1-\alpha$,
we aim to identify the smallest possible prediction set (a subset of these classes)
ensuring the correct class is within the set with a probability of at least $1 -\alpha$.
A straightforward strategy to achieve this objective involves sequentially incorporating classes from the highest to the lowest probabilities until their cumulative sum
exceeds the threshold of $1-\alpha$. Despite the network's output adopting a mathematical distribution format, it does not inherently reflect the actual class distribution. Typically,
the network will not be calibrated and it tends to be overly optimistic \citep{Guo2017}.
Consequently, this straightforward approach doesn't assure the inclusion of the correct class with the desired probability.

The first step of the CP algorithm involves forming a conformity score $S(x,y)$ that  measures the network's uncertainty between $x$ and its true label $y$ (larger scores indicate worse agreement). The Homogeneous Prediction Sets (HPS) score \citep{vovk2005conformal} is
$S_{\scriptscriptstyle \textrm HPS}(x,y)=1-p(y|x;\theta)$, s.t. $\theta$ is the network parameter set.
The Adaptive Prediction  Score (APS) \citep{romano2020classification}  is the sum of all class probabilities that are not lower  than the probability of the true class:
\begin{equation}
 S_{\scriptscriptstyle APS}(x,y) =  \sum_{\{i|p_i \ge p_{y}\}} p_i,
   \label{aps_score}
  \end{equation}
such that $p_i= p(y=i|x;\theta)$ and $p_y$ is the probability of the  label $y$.
The RAPS score \citep{angelopoulos2020uncertainty} is a variant of APS, which  is  defined as follows:
\begin{equation}
 S_{\scriptscriptstyle RAPS}(x,y) =  \sum_{\{i|p_i \ge p_{y}\}} p_i + a \cdot \max(0,(NC - b))
      \label{raps_score}
   \end{equation}
s.t. $NC=|\{i|p_i \ge p_{y}\}|$ and  $a,b$ are parameters that need to be tuned.
RAPS is especially effective in the case of a large number of classes where it explicitly encourages small prediction sets.

We can also define a randomized version of a conformity score. For example in the case of  APS
we define: \begin{equation}
 S_{\scriptscriptstyle APS}(x,y,u) =  \sum_{\{i|p_i > p_{y}\}} p_i + u \cdot p_y,\hspace{0.3cm} u\sim U[0,1].
  \label{radaps_score}
  \end{equation}
The random version tends to yield the required coverage more precisely and thus it produces smaller prediction sets \citep{angelopoulos2023conformal}.
The CP prediction set of a data point $x$ is defined as $C_q(x)=\{y| S(x,y) \le q\}$ where $q$ is a threshold that is found using a labeled calibration set $(x_1,y_1),...,(x_n,y_n)$. The CP theorem \citep{vovk2005conformal} states that if we set  $q$ to be the $(1\!-\!\alpha)$ quantile of the conformal scores $S(x_1,y_1),...,S(x_n,y_n)$  we can guarantee that $1\!-\!\alpha \le p( y\in C(x))$. 
where $x$ is a test point and $y$ is its the unknown true label \citep{vovk2005conformal}.
 In the random case there is still a coverage guarantee, which is defined by marginalizing over all test points $x$  and samplings  $u$ from the uniform distribution  \citep{romano2020classification}.
%Note that the coverage guarantee is for a marginal probability over all possible test points and coverage may be worse or better for different points. It can be proved that obtaining a conditional coverage guarantee is impossible \citep{foygel2021limits}.

\section{CP Calibration based on Noisy Labels}
\subsection{Setting the Threshold Given Noisy Labels}
%\subsection{Threshold Estimation with Noisy Data}
Here we show how, given a simple noise model and a known noise level, we can get the correct CP threshold based on noisy data. We will generalize this beyond the simple noise model in the following section. Consider a network that classifies an input $x$ into $k$ pre-defined classes.
Given a conformity score $S(x,y)$ and a specified coverage $1-\alpha$, the goal of the conformal prediction algorithm is to find a minimal $q$ such that $p(y\in C_q(x))\ge 1-\alpha$.
Let $(x_1,\tilde{y}_1),...,(x_n,\tilde{y}_n)$ be a calibration set with noisy labels and
let $y_i$ be the unknown correct label of $x_i$. Let $s_i=S(x_i,\tilde{y}_i)$ be the conformity score of $(x_i,\tilde{y}_i)$.
We assume that the label noise follows a uniform distribution, where with a probability of $\epsilon$,  the correct label is replaced by a label that is randomly sampled from  the $k$ classes:
 \begin{equation}
 p(\tilde{y}=j| y=i) = \mathds{1}_{\{i=j\}}(1-\epsilon) + \frac{\epsilon}{k}.
 \label{nnoise}
 \end{equation}
Uniform noise is relevant, for example, when applying differential privacy techniques to overcome privacy issues \citep{ghazi2021}.
In that setup the noise level $\epsilon$ is usually known. 
In other applications such as medical imaging, where the noise parameter $\epsilon$ is not given, it can be estimated with sufficient accuracy from the noisy-label data during training \citep{zhang2021learning, li2021provably, Lin2023Holistic}.
 We can write $\tilde{y}$ as $\tilde{y}=(1-z)\cdot y+z \cdot u$,
where $u$ is a random label uniformly sampled from $\{1,...,k\}$ and
  $z$ is a binary random variable $(p(z=1)=\epsilon)$ indicating whether the label of the sample $(x,y)$ was replaced by a random label or not.
For each candidate threshold, $q$ denote:
\begin{equation}
F^c(q) = p(y\in C_q(x)),  \hspace{1cm} F^n(q) = p(\tilde{y}\in C_q(x)), 
  \hspace{1cm}
F^r(q) = p(u\in C_q(x)), 
\nonumber
\end{equation}
where $F^c$, $F^n$, and $F^r$ represent the clean, noisy and random labels. Note as well that each one is the CDF of the appropriate conformal score function, e.g., $F^c(q) = p(y\in C_q(x))=p(S(x,y)\leq q)$.

It is easily verified that
  \begin{equation}
  F^n(q) =  p(z=0)F^c(q) + p(z=1) F^r(q)   = (1 - \epsilon )F^c(q) + \epsilon F^r(q).
    \label{eqAB}
\end{equation}

For each value $q$, we can estimate $F^n(q)$ from the noisy calibration set:
 \begin{equation}
\hat{F}^n(q) = \frac{1}{n} \sum_i \mathds{1}_{\{\tilde{y}_i\in C_q(x_i)\}}
          = \frac{1}{n}\sum_i \mathds{1}_{\{s_i \le q \}}.
\label{hatA}
   \end{equation}
   Note that $q$ is the $\hat{F}^n(q)$-quantile of $s_1,...,s_n$.
Similarly we can also estimate  $F^r(q)$:
 \begin{equation}
\hat{F}^r(q) = \frac{1}{n} \sum_i p ( u_i \in C_q(x_i) ) =
\frac{1}{n} \sum_i  \frac {|C_q(x_i)|}{k},
%\mathds{1}_{\{{u_i}\in C_q(x_i)\}}
\label{hatD}
 \end{equation}
where $u_i$ is  uniformly sampled from $\{1,...,k\}$.
%The general CP theory implies that:
%$$\hat{A} \le p(\tilde{y}\in C_q(x)) \le  \hat{A} + \frac{1}{n+1}, \hspace{1cm}  \hat{D} \le p({y}_{wr}\in C_q(x))  \le \hat{D} + \frac{1}{n+1}$$
By substituting  (\ref{hatA}) and (\ref{hatD}) in  (\ref{eqAB}) we obtain an estimation of $F^c(q) = p ( y\in C_q(x))$ based on the noisy calibration set and the noise level $\epsilon$:
 \begin{equation}
  \hat{F}^c(q) = \frac{\hat{F}^n(q) - \epsilon  \hat{F}^r(q)}{1-\epsilon}.
  \label{hatB}
\end{equation}

 For each candidate  $q$ we first compute $\hat{F}^n(q)$ and $\hat{F}^r(q)$ and
 then by using (\ref{hatB})  obtain the coverage estimation $\hat{F}^c(q)$. Given a coverage requirement $(1-\alpha)$,  we can thus use the noisy calibration set to find a threshold $q$ such that $\hat{F}^c(q)=1-\alpha$. Note that since $F^c(q)$ is monotonous, it seems reasonable to search for the threshold $q$ using the bisection method. However, as $\hat{F}^c(q)$ is an approximation based on the \emph{difference} between two monotonic functions, it might not be exactly monotonous. We therefore find the threshold $q$ using an exhaustive grid search.
If there are several solutions we select the largest value.
(In practice selecting one of the solutions has almost no effect on the results.)  We note that even with an exhaustive search the entire runtime is negligible compared to the training time.
We can narrow the threshold search domain  as follows:
 \begin{lemma}
For every threshold $q$ we have: $\hat{F}^n_q/k \le \hat{F}^r(q)$.
\end{lemma}
\begin{proof}
Denote $A=\{i|\hat{y}_i\in C_q(x_i)\}$ and $B=\{i| 1\le |C_q(x_i)|\}$. Note that  $\hat{F}^n(q) = |A|/n$.
$$ |B| =\sum_{i \in B} 1  \le \sum_{i\in B} |C_q(x_i)| \le \sum_{i=1}^n |C_q(x_i)|
 =nk \hat{F}^r(q).$$
Finally  $A \subset B$ implies that:
$ \hat{F}^n(q) = |A|/n \le |B|/n \le k \hat{F}^r(q). $
\end{proof}
 \begin{theorem}
Let $q_1$ be the $(1\!-\!\alpha)(1-\epsilon)/(1-\frac{\epsilon}{k})$  quantile of  $s_1,...,s_n$ and let $q_2$ be the
$(1-\alpha)+\alpha\epsilon$ quantile.
If  $q$ satisfies  $\hat{F}^c(q)=1-\alpha$ then
 $q_1\le q\le q_2$.
\end{theorem}
\begin{proof} %We will prove the upper bound here.
%, the proof for the lower bound is in the supplementary material.
 Assume $q$  satisfies
$\hat{F}^c({q})=1-\alpha$. Eq. (\ref{hatB}) implies that
  \begin{equation}
  1-\alpha =  \hat{F}^c({q})=\frac{\hat{F}^n({q}) - \epsilon  \hat{F}^r({q})}{1-\epsilon } 
  \Rightarrow \,\,
     \hat{F}^n({q})  = (1-\alpha) (1-\epsilon) + \epsilon\hat{F}^r({q}).
    \label{hatcx}
\end{equation}
  Since $0 \le \hat{F}^r(q) \le 1$ we get that:
  \begin{equation}
  (1-\alpha)(1-\epsilon) \le \hat{F}^n({q}) \le  (1-\alpha)+\alpha\epsilon =\hat{F}^n(q_2).
  \label{aq}
    \end{equation}
% For each noisy validation-set point $(x_i,\hat{y}_i)$
%if $\hat{y}_i\in C_q(x_i)$ then obviously   $|C_q(x_i)|\ge 1$.
For every $q$ we have  $ \hat{F}^n(q)/k \le \hat{F}^r(q)$ (Lemma 3.1).
Hence, $ (1-\alpha)(1-\epsilon) \le \hat{F}^n(q)$  (\ref{aq})    implies that $ (1-\alpha)(1-\epsilon)/k \le   \hat{F}^r(q) $.
Combining this inequality with Eq. (\ref{hatcx}) yields a better lower bound:
$(1-\alpha) (1-\epsilon)(1+\epsilon/k) \le \hat{F}^n(q)   $.
Iterating  this process yields:
$$
   (1-\alpha) (1-\epsilon)\left(1+\frac{\epsilon}{k} + \left(\frac{\epsilon}{k}\right)^2 + \dots   \right)
 = (1-\alpha) \frac{1-\epsilon}{1-\frac{\epsilon}{k}} = \hat{F}^n(q_1) \le \hat{F}^n(q).$$
Finaly, $\hat{F}^n(q)$ is a monotonically increasing function of $q$  which implies that $q_1 \le q\le q_2$.
\end{proof}

As an alternative to the grid search we can sort the noisy conformity scores $s_i=S(x_i,\tilde{y}_i)$ and look for the minimal $i$ such that $\hat{F}^c(s_i)\ge 1-\alpha$. In the 
noise-free case  $\hat{F}^c$ is piece-wise constant, with jumps determined exactly by the order statistics $s_i$, namely,  $\hat{F}^c(s_i)=i/n$ and thus this algorithm coincides with the standard CP algorithm. 
We dub our algorithm Noise-Aware Conformal Prediction (NACP), and summarize it in Algorithm 1.
Note that in the noise-free case ($\epsilon=0$) the NACP algorithm coincides with the standard CP algorithm and selects $q$ that satisfies $\hat{F}^c(q)=\hat{F}^n(q)=1-\alpha$, i.e., $q$ is the $1-\alpha$ quantile of the calibration set conformity scores.

\subsection{Prediction Size Comparison}

We next compare our NACP approach analytically to Noisy-CP \citep{einbinder2022conformal} in terms of the average size of the prediction set.

\begin{theorem}
Let $q$ and $\tilde{q}$ be the thresholds computed by the NACP and the Noisy-CP algorithms respectively.
Then $q \le \tilde{q}$  if and only if $\hat{F}^r(\tilde{q})  \le (1-\alpha)$.
\label{theorem3-2}
\end{theorem}
\begin{proof}
The threshold $\tilde{q}$ computed by the Noisy-CP
algorithm (by applying standard CP on the noisy validations set) satisfies $\hat{F}^n(\tilde{q})  = (1-\alpha)$.
The true threshold $q$ satisfies
$\hat{F}^n({q})  = (1-\alpha) (1-\epsilon) + \epsilon\hat{F}^r({q}) $
(\ref{hatcx}).  Looking at the difference
\begin{equation}\label{eq:diff}
   \hat{F}^n(\tilde{q})-\hat{F}^n({q})=1-\alpha-(1-\alpha)(1-\epsilon)-\epsilon \hat{F}^r({q})
   =\epsilon(1-\alpha-\hat{F}^r({q})).
    \end{equation}
Hence from the monotonicity of $\hat{F}^n(q)$ we have  $q \le \tilde{q}$ iff $ \hat{F}^n(q) \le \hat{F}^n(\tilde{q})$ iff $\hat{F}^r({q})\le 1-\alpha$.
 \end{proof}
The theorem above states that if the size of the prediction set obtained by NACP is less than $k(1-\alpha)$,  NACP is more effective than Noisy-CP.
For example, assume $k=100$ and $1-\alpha=0.9$. In this case, if the average size of the NACP prediction set is less than 90, NACP is more effective than Noisy-CP. We also see from eq. (\ref{eq:diff}) that the smaller $\hat{F}^r$ is the larger the gap between the two methods. Since  $\hat{F}^r$ is inversely proportional to the number of classes, we expect the difference to be substantial when there is a large number of classes to consider, which is exactly where CPs' ability to reliably exclude possible classes is very useful. 

\begin{algorithm}[t]
\caption{Noise-Aware Conformal Prediction (NACP) for uniform noise}
       \begin{algorithmic}[1]
      \State Input: A conformity score $S(x,y)$, a coverage level $1\!-\!\alpha$ and a calibration set $(x_1,\tilde{y}_1),...,(x_n,\tilde{y}_n)$,
      where the labels are corrupted by a uniform noise with parameter~$\epsilon$.
\State Set $q_1$ to be the $(1\!-\!\alpha)(1-\epsilon)/(1-\frac{\epsilon}{k})$ quantile of  $S(x_1,\tilde{y}_1),...,S(x_n,\tilde{y}_n)$ and set $q_2$ to be
$((1-\alpha)+\alpha\epsilon)$ quantile.
    \State For each candidate threshold $q$ compute:
$$
\hat{F}^n(q)  = \frac{1}{n} \sum_i \mathds{1}_{\{\tilde{y}_i\in C_q(x_i)\}},  \hspace{0.5cm}
\hat{F}^r(q)  = \frac{1}{n} \sum_i  \frac {|C_q(x_i)|}{k},  \hspace{0.5cm}
 \hat{F}^c(q) = \frac{\hat{F}^n(q) - \epsilon \hat{F}^r(q)}{1-\epsilon } 
$$
       \State Apply a grid search to find $q\in[q_1,q_2]$ that satisfies $\hat{F}^c(q)=1\!-\!\alpha$.
    \State The prediction set of a test sample $x$
         is: $$C_q(x)=\{y\,|\,S(x,y) \le q\}.$$
         \State Coverage guarantee:  $p ( y \in C_q(x) ) \ge 1-\alpha - \Delta(n,\epsilon,\delta)$  with probability $(1-\delta)$ over the noisy calibration set sampling  (see Theorem \ref{theorem3}).
   \end{algorithmic}
   \label{alg1}
       \end{algorithm}

\subsection{Coverage Guarantees}

We next provide a coverage guarantee for NACP. We show that if we apply the NACP to find a threshold $q$ for $1-\alpha+\Delta$, then  $P(y\in C_q(x))\ge 1-\alpha$ were $\Delta$ depends on the calibration set size. $\Delta$ is a finite-sample term that is needed to approximate the CDF to set the threshold instead of simply picking a predefined quantile.  Because $\Delta$ can be computed, one can adjust the $\alpha$ used in the NACP algorithm to get the desired coverage guarantee.  We note that we empirically found this bound to be over-conservative and that the un-adjusted method does reach the desired coverage (see Section 5).

\begin{lemma}
    Given $\delta>0$, define $\Delta=\sqrt{\frac{\log(4/\delta)}{2nh^2}}$ such that $h=\frac{1-\epsilon}{1+\epsilon}$
    and $n$ is the size of the noisy calibration set. Then 
    \begin{equation}
        p( \sup_q |F^c(q)-\hat{F}^c(q)| >  \Delta) \le \delta,
    \end{equation}
    such that the probability is over the sampling of the noisy calibration set.
\end{lemma}

\begin{proof}
    The Dvoretzky–Kiefer–Wolfowitz (DKW) inequality \citep{massart1990tight} states  that if we estimate a CDF $F$ from $n$ samples using the empirical CDF $F_n$ then $p(\sup_x|F_n(x)-F(x)|>\Delta)\leq 2\exp(-2n\Delta^2$). Eq. (\ref{hatB}) defines $\hat{F}^c(q)$ using $\hat{F}^n(q)$ and $\hat{F}^r(q)$. Both are empirical CDF, so from the DKW theorem and the union bound we get that:
    $$
    p(\sup_q|F^r(q)-\hat{F}^r(q)|>h\Delta \,\, \mbox{or}  \,\, \sup_q|F^n(q)-\hat{F}^n(q)|>h\Delta)
    \leq 4\exp(-2nh^2\Delta^2)= \delta. 
   $$ 
Using eq. (\ref{hatB}) we get that with probability at least $1-\delta$ for every $q$:
    \begin{align}
        &\hat{F}^c(q)= \frac{\hat{F}^n(q) - \epsilon  \hat{F}^r(q)}{1-\epsilon }\leq   \frac{(F^n(q)+h\Delta) - \epsilon  (F^r(q)-h\Delta)}{1-\epsilon } \nonumber 
         F^c(q)+\frac{h\Delta + \epsilon h\Delta}{1-\epsilon } \\ & = F^c(q)+h\Delta\frac{1+\epsilon}{1-\epsilon}= F^c(q)+\Delta. \nonumber
    \end{align}
    Similarly, we can show that $\hat{F}^c(q)\geq F^c(q)-\Delta$ which completes the proof.
\end{proof}

The proof of the main theorem now follows the standard CP proof, taking the inaccuracy in estimating $F^c(q)$ into account.

\begin{theorem}
    Assume you have a noisy calibration set of size $n$ with noise level $\epsilon$ and set $\Delta(n,\epsilon,\delta)=\sqrt{\frac{\log(4/\delta)}{2nh^2}}$ where $h=\frac{1-\epsilon}{1+\epsilon}$ and that you pick $q$ such that $\hat{F}^c(q) \ge 1-\alpha+\Delta$. Then with probability at least $1-\delta$ (over the calibration set), we have that if $(x,y)$ are sampled from the clear label distribution we get:
    $$ 1-\alpha \le p(y\in C_q(x)).$$
    \label{theorem3}
\end{theorem}

\begin{proof}
    Given a clean test pair $(x,y)$, with probability $\delta$ over the calibration set,  we have: $$p(y\in C_q(x))=p(S(x,y)<q) =F^c(q)\geq \hat{F}^c(q)-\Delta \geq  1-\alpha.$$
%    In a similar way: $p(y\in C_q(x))=F^c(q)\leq \hat{F}^c(q)+\Delta = 1-\alpha + 2\Delta$.
\end{proof}
As the size of the noisy calibration set, $n$, tends to infinity, $\Delta$ converges to  zero and thus the noisy threshold converges to the noise-free
threshold.  

\subsection{A More General Noise Model}
\label{sec:general_noise_model}
The focus of this paper is on the case of a uniform noise. Our approach, however, is easily extended to a more general noise model. We will assume that the noisy label $\tilde{y}$ is independent of $x$ given $y$. We also assume that the noise matrix  $P(i,j)=p(\tilde{y}=j|y=i)$ is known and that the matrix P is invertible. For each $q$ define the following matrices for the clear and the noisy data:
 ${M}_q^c(\ell,i)=p( \ell\in C_q(x),y=i)$ and
 ${M}_q(\ell,i)=p( \ell\in C_q(x),\tilde{y}=i)$. Assuming that, given the
 true label $y$,  the r.v. $x$ and $\tilde{y}$ are independent, we obtain:
    \begin{align}
        {M}_{q}(\ell,i) & =   p(\ell\in C_q(x),\tilde{y}=i)  =
        \sum_j p(\ell\in C_q(x),\tilde{y}=i,y=j) \label{generalP}\\&=\sum_j p(\ell\in C_q(x),y=j)p(\tilde{y}=i|y=j) = \sum_j {M}^c_{q}(\ell,j)P(j,i). \nonumber
    \end{align}
We can write (\ref{generalP}) in matrix notation: $ {M}_q=M^c_q P$. Eq. (\ref{generalP}) implies that:
\begin{equation}
F^c(q)= p ( y\in C_q(x)) = \sum_i p(i\in C_q(x) , y=i)=
\sum_i M^c_q(i,i) = \textrm{Tr} ( M_qP^{-1}).
\label{estM1}
\end{equation}
We can estimate the matrix  $M_q$ from the noisy samples:
\begin{equation}
    \hat{M}_{q}(\ell,i)  = \frac{1}{n}\sum_j \mathds{1}_{\{ \tilde{y}_j = i, \,\, \ell \in C_q(x_j) \}},\hspace{1cm} i,\ell=1,..,k.
    \label{estM2}
\end{equation}
Substituting (\ref{estM2}) in (\ref{estM1}) yields  an estimation of the  probability $F^c(q)=p(y\in C_q(x))$:
\begin{equation}
\hat{F}^c(q) = \textrm{Tr}(\hat{M}_{q} P^{-1}).
\label{estM3}
\end{equation}
The final step is applying a grid search to find a threshold $q$ such that $\hat{F}^c(q)=1-\alpha$.

In the case that $P$ is a uniform noise matrix (\ref{nnoise}),
the Sherman-Morison formula implies that $P^{-1}=(\frac{1}{1-\epsilon}I - \frac{\epsilon}{(1-\epsilon)k} \1 \1^{\T})$. Therefore,
$$
\hat{F}^c(q) = \textrm{Tr}(\hat{M}_{q} P^{-1})=
\frac{1}{1-\epsilon}\sum_i \hat{M}_{q}(i,i)
- \frac{\epsilon}{(1-\epsilon)k}  \sum_{\ell,i} \hat{M}_{q}(\ell,i) = \frac{\hat{F}^n(q) - \epsilon \hat{F}^r(q)}{1-\epsilon }.
$$
Thus in the case of a uniform noise the coverage estimation (\ref{estM3}) coincides with  (\ref{hatB}).

  \section{Related Work}
In this section we review two closely related works that address the same problem of calibration with noisy labels \citep{sesia2023adaptive, Clarkson2024}. 
The derivation of the noisy conformal threshold in these two works is similar to ours. 
These two methods compute the same threshold $q$ that satisfies $\hat{f}^c(q)=1-\alpha$ (\ref{estM3}).  The only minor difference is that in these two studies they use the distribution of correct labels given the noisy labels, while we use the more natural distribution of the noisy labels given the correct label. As a result, they need to know the marginal class frequencies for both the clean and noisy labels, whereas we do not. 
Each one of the two methods provides a different finite coverage guarantee 
in the form of: $$ p(y \in C_q(x)) \ge 1-\alpha-\Delta$$ where $\Delta$ depends on the calibration set size $n$, the number of classes $k$, and the noise model, but it doesn't depend on the validation dataset itself. 

We first review the bound $\Delta$ derived in \citep{sesia2023adaptive}.
 Let $\rho_i = p(y=i) $ and $\tilde{\rho}_i=p(\tilde{y}=i)$ be the marginal true and noisy label distributions.   Let $M(y|\tilde{y})$ be the noise  conditional distribution and let $V=M^{-1}$.
 Let $c(n) = \mathbb{E} \left[ \max_{i \in [n]} \left( \frac{i}{n} - u_{(i)} \right) \right]$, such that $\{u_{(i)}\}_{i=1}^{n}$ order statistics of $\{u_{i}\}_{i=1}^{n}$ i.i.d. uniform random variables on $[0, 1]$. The size of the least common class is $n_*=\min_{i\in[k]} n_i$ where $n_i$ is the number of samples of noisy label $i$. Finally, the finite sample correction is:
 \begin{align}
     \Delta&=c(n) + \frac{2 \max_{i \in [k]} \sum_{l \neq i} |V_{il}| + 
 \sum_{i=1}^{k} |\rho_i - \tilde{\rho}_i|}{\sqrt{n_*}} \\ &  \cdot \min \left(k^2 \sqrt{\frac{\pi}{2}}, 
 \frac{1}{\sqrt{n^*}} + \sqrt{\frac{\log(2k^2) + \log(n^*)}{2}} \right). \nonumber
 \end{align}
It can be easily verified that $\Delta=O(\log{k})$ and therefore the bound becomes less effective for large values of $k$.

The finite sample term $\Delta$ derived in \citep{Clarkson2024} is: 
\begin{equation}
\Delta =  \sum_{i=1}^k  ( |w_i^{(1)}|b(n, i) + \sum_{i \neq j} |w_{ij}^{(2)}| b(n, j)  ) 
\label{clarkson}
\end{equation}
where
$k$ is the number of classes, $w^{(1)}_i = P_{i,i}^{-1}\rho_i - \tilde{\rho}_i$, $w^{(2)}_{ij} = \rho_i P^{-1}_{ji}$, and 
$ b(n, j) = (1 - \tilde{\rho}_j)^n + \sqrt{\frac{
\pi }{n \tilde{\rho}_j}} $.  $\rho_i = p(y=i)$ and $ \tilde{\rho}_i = p(\tilde{y}=i)$  are the marginal clean and contaminated label probabilities and $P_{ji} =p(y = j | \tilde{y} = i) $ is the conditional label noise distribution.
It can be easily verified that $\Delta=O(\sqrt{k})$ and therefore the bound becomes less effective for large values of $k$.

We note that our finite sample term $\Delta$ (see Lemma 4) does not depend on the number of classes $k$. Therefore, unlike the algorithms of  \citep{sesia2023adaptive} and \citep{Clarkson2024}, it remains effective even in tasks with many classes.
A further distinction between our approach and previous works is that their finite sample coverage guarantee is established based on the average of all the noisy calibration sets. In contrast, our approach provides an individual coverage guarantee for nearly all ($1-\delta$) of the sampled noisy calibration sets. In Section 5  we show that the average coverage guarantee 
obtained by \citep{sesia2023adaptive} and \citep{Clarkson2024} implies that in tasks with a large number of classes, the prediction set should include all the classes and therefore it is useless. In contrast, our individual finite set coverage guarantee, on to $(1-\delta)$ portion of the noisy calibration sets, remains effective for tasks with many classes.

We note that our observation that the finite sample correction  term  does not depend on the number of classes, applies to the case of a uniform label noise. In the case of a general noise matrix, all finite sample correction terms are not effective.

\section{Experiments}

In this section, we evaluate the capabilities of our NACP algorithm on various publicly available image classification datasets. 

 \textbf{Compared methods.}   We  implemented  three popular conformal prediction scores, namely APS
\citep{romano2020classification}, RAPS
\citep{angelopoulos2020uncertainty}
 and HPS \citep{vovk2005conformal}. For each score, we implemented the following CP  methods: (1) CP (oracle) -  using a calibration set with clean labels,  (2) Noisy-CP -  applying a standard CP on noisy labels without any modifications \citep{einbinder2022conformal},
 (3)  NR-CP (w/o $\Delta$) - Noise-Robust CP approach without the finite sample coverage guarantee $\Delta$,  
see Eq. (\ref{estM3}) and \citep{sesia2023adaptive,Clarkson2024}. 
  We also implemented three methods that add finite sample coverage guarantee terms
  to the NR-CP method. 
 (4)   Adaptive Conformal Classification with Noisy labels (ACNL) \citep{sesia2023adaptive}, (5) Contamination Robust Conformal Prediction (CRCP) \citep{Clarkson2024}, and (6) NACP - our approach.
 For methods (4), (5), we used their official codes 
\footnote{ \url{https://github.com/msesia/conformal-label-noise}} \footnote{ \url{https://github.com/jase-clarkson/cp_under_data_contamination}} 
 and we share our code for reproducibility\footnote{\url{https://github.com/cobypenso/Noise-Aware-Conformal-Prediction}}.

\begin{table*}[t]
\centering
\caption{ CP calibration results  for $1\!-\!\alpha$ = 0.9
and noise level $\epsilon=0.2$.  We report the mean and the std over 1000 different splits. We show the  best result with theoretical guarantees in bold.}
\label{raps2}
   \scalebox{.7}{
\begin{tabular}{ll|rr|rr|rr}
& & \multicolumn{2}{c|}{APS}  & \multicolumn{2}{c|}{RAPS}  & \multicolumn{2}{c}{HPS}\\
\hline
Dataset      &  CP Method & size $\downarrow$ &  coverage(\%)   &  size $\downarrow$ &  coverage(\%) & size $\downarrow$ &  coverage(\%)   \\ \hline

                      &  CP (oracle) &  1.1 $\pm$ 0.01 & 90.0 $\pm$ 0.62 & 1.1 $\pm$ 0.01 & 90.0 $\pm$ 0.61  &  0.9 $\pm$ 0.01 & 90.0 $\pm$ 0.59  \\
                      &  Noisy-CP &  5.1 $\pm$ 0.18 & 99.9 $\pm$ 0.04 & 5.1 $\pm$ 0.18 & 99.9 $\pm$ 0.04  &  5.1 $\pm$ 0.18 & 99.8 $\pm$ 0.04  \\
                        &  NR-CP (w/o $\Delta$) &  \  {1.1 $\pm$ 0.02} & \  {90.1 $\pm$ 0.70} & \  {1.1 $\pm$ 0.02} & \  {90.1 $\pm$ 0.69}  &  \  {0.9 $\pm$ 0.02} & \  {90.0 $\pm$ 0.75}  \\
 CIFAR-10  &  ACNL &  1.5 $\pm$ 0.06 & 96.0 $\pm$ 0.61 & 1.3 $\pm$ 0.03 & 94.6 $\pm$ 0.65  &  1.1 $\pm$ 0.03 & 96.0 $\pm$ 0.59  \\
(10 classes)                      &  CRCP &  \textbf{1.2 $\pm$ 0.03} & \textbf{93.7 $\pm$ 0.62} & \textbf{1.2 $\pm$ 0.03} & \textbf{93.7 $\pm$ 0.62}  &  \textbf{1.1 $\pm$ 0.01} & \textbf{95.7 $\pm$ 0.18}  \\
                      &  NACP &  1.3 $\pm$ 0.04 & 94.4 $\pm$ 0.62 & 1.3 $\pm$ 0.04 & 94.5 $\pm$ 0.62  &  1.1 $\pm$ 0.01 & 95.9 $\pm$ 0.18  \\
                     \hline
                  
                              &  CP (oracle) &  6.5 $\pm$ 0.20 & 90.0 $\pm$ 0.43 & 4.0 $\pm$ 0.08 & 90.0 $\pm$ 0.43  &  2.0 $\pm$ 0.03 & 90.0 $\pm$ 0.43  \\
                              &  Noisy-CP &  50.5 $\pm$ 1.29 & 99.8 $\pm$ 0.04 & 50.5 $\pm$ 1.33 & 99.8 $\pm$ 0.03  &  50.1 $\pm$ 1.34 & 99.9 $\pm$ 0.02  \\
                                &  NR-CP (w/o $\Delta$) &  {6.4 $\pm$ 0.28} & \  {89.9 $\pm$ 0.54} & {4.0 $\pm$ 0.11} & \  {89.9 $\pm$ 0.55}  &  \  {2.0 $\pm$ 0.06} & \  {89.9 $\pm$ 0.56}  \\ 
    CIFAR-100        &  ACNL &  100.0 $\pm$ 0.00 & 100.0 $\pm$ 0.00 & 100.0 $\pm$ 0.00 & 100.0 $\pm$ 0.00   &  100.0 $\pm$ 0.00 & 100.0 $\pm$ 0.00  \\
(100 classes)                              &  CRCP &  25.7 $\pm$ 3.71 & 98.7 $\pm$ 0.39 & 8.5 $\pm$ 0.41 & 98.3 $\pm$ 0.16  &  11.1 $\pm$ 3.46 & 98.7 $\pm$ 0.42  \\
                              &  NACP &  \textbf{9.0 $\pm$ 0.46} & \textbf{93.0 $\pm$ 0.49} & \textbf{4.8 $\pm$ 0.13} & \textbf{93.0 $\pm$ 0.48}  &  \textbf{2.5 $\pm$ 0.09} & \textbf{93.0 $\pm$ 0.52}  \\
                            \hline
                              &  CP (oracle) &  14.9 $\pm$ 0.60 & 90.0 $\pm$ 0.61 & 6.9 $\pm$ 0.19 & 90.0 $\pm$ 0.62  &  3.8 $\pm$ 0.13 & 90.02 $\pm$ 0.58  \\
                              &  Noisy-CP &  99.7 $\pm$ 3.67 & 99.7 $\pm$ 0.08 & 101.4 $\pm$ 3.58 & 99.5 $\pm$ 0.09  &  98.3 $\pm$ 3.80 & 99.8 $\pm$ 0.05  \\
                                &  NR-CP (w/o $\Delta$) &  \  {14.0 $\pm$ 0.91} & \  {89.4 $\pm$ 0.81} & {6.7 $\pm$ 0.27} & \  {89.3 $\pm$ 0.80}  &  \  {3.5 $\pm$ 0.24} & \  {89.3 $\pm$ 0.80}  \\ 
                 TinyImagenet        &  ACNL & 200.0 $\pm$ 0.00 & 100.0 $\pm$ 0.00 & 200.0 $\pm$ 0.00 & 100.0 $\pm$ 0.00  &  200.0 $\pm$ 0.00 & 100.0 $\pm$ 0.00  \\
(200 classes)                              &  CRCP &  200.0 $\pm$ 0.00 & 100.0 $\pm$ 0.00 & 200.0 $\pm$ 0.00 & 100.0 $\pm$ 0.00 &  200.0 $\pm$ 0.00 & 100.0 $\pm$ 0.00  \\
                              &  NACP &  \textbf{22.6 $\pm$ 1.87} & \textbf{93.7 $\pm$ 0.71} & \textbf{9.0 $\pm$ 0.50} & \textbf{93.6 $\pm$ 0.70} &  \textbf{7.0 $\pm$ 0.87} & \textbf{93.6 $\pm$ 0.72}  \\
                            \hline
                              &  CP (oracle) &  16.6 $\pm$ 0.33 & 90.0 $\pm$ 0.26 & 6.3 $\pm$ 0.06 & 90.0 $\pm$ 0.27  &  3.6 $\pm$ 0.07 & 90.0 $\pm$ 0.28  \\
                              &  Noisy-CP &  502.6 $\pm$ 8.56 & 99.9 $\pm$ 0.01 & 501.6 $\pm$ 8.51 & 99.9 $\pm$ 0.01  &  501.3 $\pm$ 10.2 & 100.0 $\pm$ 0.01  \\
                                                  &  NR-CP (w/o $\Delta$) &  {16.7 $\pm$ 0.51} & \  {90.0 $\pm$ 0.34} & {6.3 $\pm$ 0.10} & \  {90.0 $\pm$ 0.36}  &  {3.6 $\pm$ 0.14} & \  {90.0 $\pm$ 0.38}  \\
                  ImageNet        &  ACNL &  1000.0 $\pm$ 0.00 & 100.0 $\pm$ 0.00 & 1000.0 $\pm$ 0.00 & 100.0 $\pm$ 0.00  &  1000.0 $\pm$ 0.00 & 100.0 $\pm$ 0.00  \\
(1000 classes)                              &  CRCP &  1000.0 $\pm$ 0.00 & 100.0 $\pm$ 0.00 & 1000.0 $\pm$ 0.00 & 100.0 $\pm$ 0.00  &  1000.0 $\pm$ 0.00 & 100.0 $\pm$ 0.00  \\
                              &  NACP &  \textbf{20.9 $\pm$ 0.72} & \textbf{91.9 $\pm$ 0.32} & \textbf{7.1 $\pm$ 0.13} & \textbf{91.9 $\pm$ 0.34}  &  \textbf{4.8 $\pm$ 0.23} & \textbf{91.9 $\pm$ 0.36}  \\
           \hline
\end{tabular}
    }
\label{many_classes_datasets}
\end{table*}

\textbf{Evaluation Measures}.  We evaluated each CP method   based on the average size of the prediction sets (where a small value means high efficiency) and the fraction of test samples for which the prediction sets contained the ground-truth labels. The two evaluation metrics are formally defined as:
$$ \textrm{size} = \frac{1}{n} \sum_i | C(x_i) |,  \hspace{0.3cm}
 \textrm{coverage} = \frac{1}{n} \sum_i
{\textbf{1}}(y_i \in C(x_i))$$
such that  $n$ is the size of the test set.  We report the mean and standard deviation over 1000 random splits.

 \textbf{Datasets.}
We show results on four standard scenery image datasets {CIFAR-10},
 {CIFAR-100}
 \citep{krizhevsky2009learning}, {Tiny-ImageNet}, and ImageNet \citep{deng2009imagenet}.

\textbf{Implementation details.} Each task was trained by fine-tuning on a pre-trained ResNet-18 \citep{he2016deep} network. The models were taken from the PyTorch site\footnote{
\url{https://pytorch.org/vision/stable/models.html}}. We selected this network architecture because of its widespread use in classification problems.
The last fully connected layer output size was modified to fit the corresponding number of classes for each dataset. 
For the standard dataset evaluated in Table \ref{many_classes_datasets} we used publicly available checkpoints. For each dataset, we combined the calibration and test sets and then constructed 1000 different splits where 50\% was used for the calibration phase and 50\% was used for testing. In all our experiments we used $\delta=0.001$. In other words, the computed
coverage guarantee is applied to the sampled noisy calibration set with probability 0.999,

\begin{table}
    \caption{Finite sample correction terms $\Delta$ of NACP, ACNL \citep{sesia2023adaptive} and CRCP \citep{Clarkson2024},  for several datasets and two noise levels, $n$ is the size of the calibration set.}
    \centering
     \scalebox{.90}{
    \begin{tabular}{lrr|cc|cc|cc}
      Dataset &  $n $&\#classes &  \multicolumn{2}{c}{NACP  }  & \multicolumn{2}{c}{ACNL}  &
        \multicolumn{2}{c}{CRCP} \\
                 & & & $\epsilon=0.1$ & $ \epsilon=0.2$ &
            $\epsilon=0.1$ &
            $\epsilon=0.2$ &
            $\epsilon=0.1$ &
            $\epsilon=0.2$ \\
            \hline
            CIFAR-10 & 5000 &10& 0.035 & 0.043 & 0.031 & 0.059 & \textbf{0.016} & \textbf{0.036}\\
            CIFAR-100 & 10000 &100& \textbf{0.025} & \textbf{0.030} & 0.077 & 0.163 & 0.039 & 0.088 \\
            TinyImagenet & 5000 & 200 & \textbf{0.035} & \textbf{0.043} & 0.175 & 0.382 & 0.078 & 0.176 \\
            ImageNet & 25000 &1000  & \textbf{0.016} & \textbf{0.019} & 0.194 & 0.466 & 0.079 & 0.177\\ \hline
    \end{tabular}
           }
    \label{tab:deltas_only_cv}
\end{table}

 % \textbf{CP results on datasets with a large number of classes.}
  \textbf{Conformal prediction results.}
 Table \ref{many_classes_datasets} reports the noisy label calibration results across 3 different conformal prediction scores, HPS, APS, and RAPS for four standard publicly available datasets,  CIFAR-10, CIFAR-100, Tiny-ImageNet, and ImageNet. In all cases, we used $1\!-\!\alpha=0.9$ and a noise level of $\epsilon=0.2$. The results indicate that in the case of a calibration set with noisy labels, the Noisy-CP threshold became larger to facilitate the uncertainty induced by the noisy labels. This yielded larger prediction sets and the coverage was higher than the target coverage which was set to $90\%$. We can see that  NACP outperformed the ACNL, and CRCP methods for all datasets except for CIFAR-10 with fewer classes. Following Theorem \ref{theorem3-2}, we expect the gain in performance when using NACP versus Noisy-CP to increase with the number of classes, indeed validated empirically in Table \ref{many_classes_datasets}.  Here for CIFAR-100, Tiny-ImageNet, and ImageNet the ACNL  and CRCP  methods failed due to the large number of classes and the relatively small number of samples per class. A direct comparison of the finite sample correction terms $\Delta$ obtained by NACP, ACNL and CRCP is shown in Table  \ref{tab:deltas_only_cv}. Note that if $1-\alpha+\Delta >1$, the prediction set includes all the classes and thus it becomes useless. We can see in Table \ref{tab:deltas_only_cv} that this is the case for ACNL and CRCP in datasets with a large number of classes when the noise level is sufficiently high.

\textbf{Correction term analysis.} Following the theoretical and empirical results, the effectiveness of our method and baselines can be fully explained by the correction terms $\Delta$ each method guarantees as practitioners require coverage guarantee and therefore will use $1-\alpha+\Delta$.  Note that, as explained in Section 4,
our finite sample coverage guarantee is different from the one provided by the baseline method. Figure \ref{fig:corr_terms_analysis} presents the correction term as a function of calibration set size and the number of classes. 
Note that the NACP curve remains exactly the same across the 3 plots. 
Our main  contribution is grounded in the fact that NACP is not dependent on the number of classes $k$, clearly shown in plots as the number of classes grows.

\begin{comment}
\textbf{ImageNet with different calibration set size.}
In the following experiment, we test the performance of various conformal prediction methods under noisy labels as a function of the calibration set size on the ImageNet dataset. Figure \ref{fig:size_sweep} shows the mean size and coverage as a function of the calibration set size. In addition, the correction term $\Delta$ is depicted for ImageNet for each calibration set size. Results show that even with as little as 2500 images that correspond to 2.5 images per class the calibration results are almost on par with the oracle calibration given clean labels.

\begin{figure}[h]
    \centering
    \includegraphics[width=0.32\linewidth]{images/imagenet_sweep_on_sizes__size.png}
    \includegraphics[width=0.32\linewidth]{images/imagenet_sweep_on_sizes__coverage.png} 
    \includegraphics[width=0.32\linewidth]{images/imagenet_sweep_on_sizes__correction_terms.png} \\
    \hspace{1cm} (a) \hspace{4cm} (b) \hspace{4cm} (c)
    \caption{Noisy labels conformal prediction on ImageNet with different calibration set sizes. (a) Mean size (b) Coverage (\%), and (c) Correction terms $\Delta$ as a function of calibration set size.}
    \label{fig:size_sweep}
\end{figure}
\end{comment}

\begin{figure}
  \begin{center}
    \includegraphics[width=0.3\textwidth]{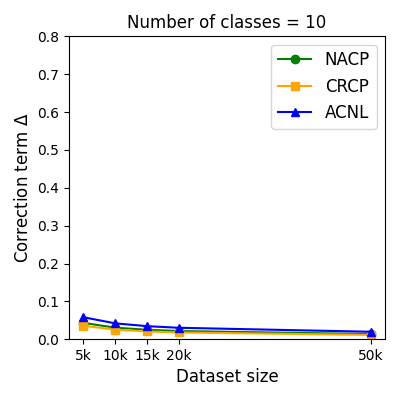}
    \includegraphics[width=0.3\textwidth]{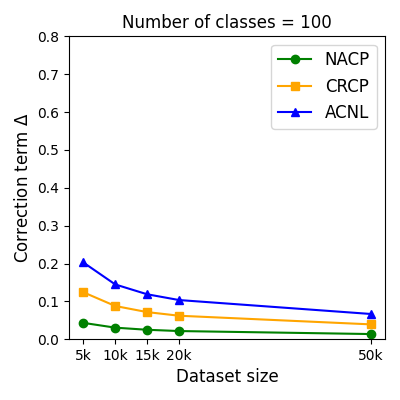}
    \includegraphics[width=0.3\textwidth]{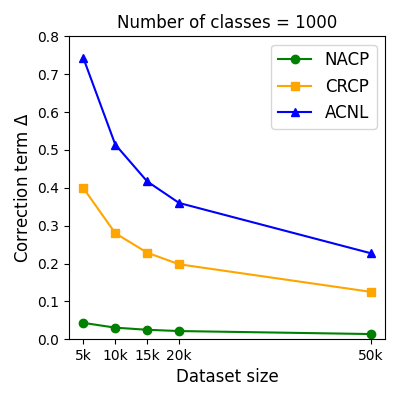}
  \end{center}
  \caption{Correction terms $\Delta$ of NACP, ACNL and CRCP as a function of the calibration set size $n$ given $\epsilon=0.2$.
  We show results for 3 numbers of classes, 10, 100 and 1000.}
  \label{fig:corr_terms_analysis}
\end{figure}

\textbf{Different network architectures.}
Conformal prediction in general and our method NACP specifically has is agnostic to the underlying model architecture.
We next experimentally verify it with ImageNet across different model architectures. Table \ref{many_archs} presents the results of applying conformal prediction on noisy labels using  ResNet18, ResNet50, DenseNet121, ViT-B16 (Vision transformer).
We can see that the NACP yields significantly improved results
regardless of the network architecture while the coverage guarantee provided by ACNL and CPRC is useless.
We can see that even without adding any correction term, we get the required coverage guarantee. This indicates that all the current correction terms are very conservative.    
\begin{table}[h]
\centering
\caption{ CP calibration results on ImageNet and various model architectures for $1\!-\!\alpha$ = 0.9
and $\epsilon=0.2$.  We report the mean and the std over 1000 different splits. Bold for best result with theoretical guarantees.}
\label{raps2appn}
   \scalebox{.6}{
\begin{tabular}{ll|rr|rr|rr|rr}
 & & \multicolumn{2}{c|}{ResNet-18} & \multicolumn{2}{c|}{ResNet-50}  & \multicolumn{2}{c|}{DenseNet121}  & \multicolumn{2}{c}{ViT-B16} \\
\hline
Dataset      &  CP Method & size $\downarrow$ &  coverage(\%)   &  size $\downarrow$ &  coverage(\%) & size $\downarrow$ &  coverage(\%)    & size $\downarrow$ &  coverage(\%)     \\ \hline

APS    &  CP (oracle) &  16.6 $\pm$ 0.33 & 90.0 $\pm$ 0.26  &  13.9 $\pm$ 0.34 & 90.0 $\pm$ 0.28   &   12.0 $\pm$ 0.28 & 90.0 $\pm$ 0.27  & 10.7 $\pm$ 0.38 & 90.0 $\pm$ 0.25   \\
            &  Noisy-CP &  502.6 $\pm$ 8.56 & 99.9 $\pm$ 0.01  &  505.5 $\pm$ 8.11 & 99.9 $\pm$ 0.01   &   502.8 $\pm$ 8.46 & 99.9 $\pm$ 0.01  & 506.8 $\pm$ 8.14 & 99.8 $\pm$ 0.02   \\
            &  NR-CP (w/o $\Delta$)     &  \  {16.7 $\pm$ 0.51} & \  {90.0 $\pm$ 0.34}  &  13.9 $\pm$ 0.47 & 90.0 $\pm$ 0.37   &   12.0 $\pm$ 0.38 & 90.0 $\pm$ 0.34  & 10.7 $\pm$ 0.55 & 90.0 $\pm$ 0.35   \\ 
            & ACNL & 1000.0 $ \pm 0.00$  & 100.0  $\pm$ 0.00   & 1000.0 $ \pm 0.00$  & 100.0  $\pm$ 0.00  & 1000.0 $ \pm 0.00$  & 100.0  $\pm$ 0.00  & 1000.0 $ \pm 0.00$  & 100.0  $\pm$ 0.00   \\
            & CRCP & 1000.0 $ \pm 0.00$  & 100.0  $\pm$ 0.00  & 1000.0 $ \pm 0.00$  & 100.0  $\pm$ 0.00  & 1000.0 $ \pm 0.00$  & 100.0  $\pm$ 0.00  & 1000.0 $ \pm 0.00$  & 100.0  $\pm$ 0.00                   \\
            &  NACP   &  \textbf{20.9 $\pm$ 0.72} & \textbf{91.9 $\pm$ 0.32}  &  \textbf{17.4 $\pm$ 0.62} & \textbf{91.9 $\pm$ 0.36}   &   \textbf{15.1 $\pm$ 0.55} & \textbf{91.9 $\pm$ 0.34}  & \textbf{15.5 $\pm$ 0.81} & \textbf{91.9 $\pm$ 0.31} \\
            \hline               
RAPS   &  CP (oracle)  &  6.3 $\pm$ 0.06 & 90.0 $\pm$ 0.27  &  4.5 $\pm$ 0.05 & 89.9 $\pm$ 0.29   &   4.7 $\pm$ 0.06 & 90.0 $\pm$ 0.26  &  2.6 $\pm$ 0.04 & 90.0 $\pm$ 0.25   \\
            &  Noisy-CP  &  501.6 $\pm$ 8.51 & 99.9 $\pm$ 0.01  &  501.1 $\pm$ 8.85 & 99.9 $\pm$ 0.01   &   501.9 $\pm$ 8.80 & 99.9 $\pm$ 0.01  & 505.8 $\pm$ 7.90 & 99.9 $\pm$ 0.01   \\
            &  NR-CP (w/o $\Delta$)    &  \  {6.3 $\pm$ 0.10} & \  {90.0 $\pm$ 0.36}  &  4.5 $\pm$ 0.06 & 90.0 $\pm$ 0.35  & 4.7 $\pm$ 0.08 & 90.0 $\pm$ 0.36  &   2.6 $\pm$ 0.05 & 90.0 $\pm$ 0.36    \\
            & ACNL & 1000.0 $ \pm 0.00$  & 100.0  $\pm$ 0.00  & 1000.0 $ \pm 0.00$  & 100.0  $\pm$ 0.00  & 1000.0 $ \pm 0.00$  & 100.0  $\pm$ 0.00  & 1000.0 $ \pm 0.00$  & 100.0  $\pm$ 0.00            \\
            & CRCP & 1000.0 $ \pm 0.00$  & 100.0  $\pm$ 0.00  & 1000.0 $ \pm 0.00$  & 100.0  $\pm$ 0.00  & 1000.0 $ \pm 0.00$  & 100.0  $\pm$ 0.00  & 1000.0 $ \pm 0.00$  & 100.0  $\pm$ 0.00                   \\
            &  NACP    &  \textbf{7.1 $\pm$ 0.13} & \textbf{91.9 $\pm$ 0.34}  &  \textbf{5.0 $\pm$ 0.08} & \textbf{91.9 $\pm$ 0.34}   &   \textbf{5.3 $\pm$ 0.10} & \textbf{92.0 $\pm$ 0.35}  &  \textbf{2.9 $\pm$ 0.07} & \textbf{92.0 $\pm$ 0.30}    \\
            \hline                
HPS         &  CP (oracle) &  3.6 $\pm$ 0.07 & 90.0 $\pm$ 0.28  &  2.0 $\pm$ 0.03 & 90.0 $\pm$ 0.28   &   2.4 $\pm$ 0.03 & 90.0 $\pm$ 0.25  &   1.5 $\pm$ 0.02 & 90.0 $\pm$ 0.26   \\
            &  Noisy-CP  &  501.3 $\pm$ 10.2 & 100.0 $\pm$ 0.01   &  502.4 $\pm$ 9.50 & 99.9 $\pm$ 0.01   &   502.3 $\pm$ 10.3 & 99.9 $\pm$ 0.20  &  504.3 $\pm$ 8.19 & 99.9 $\pm$ 0.01   \\
             &  NR-CP (w/o $\Delta$)    &  \  {3.6 $\pm$ 0.14} & \  {90.0 $\pm$ 0.38} & 2.1 $\pm$ 0.06 & 90.0 $\pm$ 0.38  &  2.4 $\pm$ 0.07 & 90.0 $\pm$ 0.34   &   1.5 $\pm$ 0.03 & 90.0 $\pm$ 0.35   \\
            & ACNL & 1000.0 $ \pm 0.00$  & 100.0  $\pm$ 0.00  & 1000.0 $ \pm 0.00$  & 100.0  $\pm$ 0.00  & 1000.0 $ \pm 0.00$  & 100.0  $\pm$ 0.00  & 1000.0 $ \pm 0.00$  & 100.0  $\pm$ 0.00                   \\
            & CRCP & 1000.0 $ \pm 0.00$  & 100.0  $\pm$ 0.00  & 1000.0 $ \pm 0.00$  & 100.0  $\pm$ 0.00  & 1000.0 $ \pm 0.00$  & 100.0  $\pm$ 0.00  & 1000.0 $ \pm 0.00$  & 100.0  $\pm$ 0.00                   \\
            &  NACP   &  \textbf{4.8 $\pm$ 0.23} & \textbf{91.9 $\pm$ 0.36}  &  \textbf{2.6 $\pm$ 0.10} & \textbf{91.9 $\pm$ 0.37}      & \textbf{3.1 $\pm$ 0.12} & \textbf{91.9 $\pm$ 0.34} & \textbf{1.7 $\pm$ 0.04} & \textbf{91.9 $\pm$ 0.33}  \\
           \hline
\end{tabular}
    }
\label{many_archs}
\end{table}

\section{Conclusions}

We presented a procedure that applies the Conformal Prediction algorithm on a calibration set with noisy labels. We first presented our method in the simpler case of a uniform noise model and then extended it to a general noise matrix. We showed that if the noise level is given,  we can find the noise-free calibration threshold without access to clean data by using the noisy-label data.  We showed that the finite sample coverage guarantee obtained by current methods is not effective in the case of classification tasks with a large number of classes. We proposed a different notion of coverage guarantee and
derived a suitable finite sample coverage guarantee that remains effective in tasks with a large number of classes.  We showed, however, that even without adding the finite sample term, noise-robust methods obtained the required coverage. This indicates that the current coverage guarantee methods are very conservative and there is room for future research to improve it and to find much tighter bounds.  In this study, we focused on noise models that assume the noisy label and the input image are independent, given the true label. In a more general noise model, the label corruption process also depends on the input features.
\bibliography{refs}

% \newpage
% \appendix
% \section{ }

 %\input{appendix}

\end{document}